\documentclass[twoside,11pt]{article}
\usepackage{jmlr2e}

\usepackage{mathtools}
\usepackage[capitalise]{cleveref}

\newcommand{\R}{\Re}
\newcommand{\cF}{\mathcal{F}}
\newcommand{\eps}{\varepsilon}
\renewcommand{\mod}{~\mathrm{mod}~}
\newcommand\growth{\Pi}
\newcommand\VCdim{{\operatorname{VCdim}}}
\newcommand\pdim{{\operatorname{Pdim}}}
\newcommand\calS{{\mathcal{S}}}
\newcommand\calX{{\mathcal{X}}}
\DeclareMathOperator{\sgn}{sgn}
\DeclareMathOperator{\vcd}{VCDim}
\newcommand{\poly}{\operatorname{poly}}
\newcommand{\polylog}{\operatorname{polylog}}

\ShortHeadings{Nearly-tight VC-dimension bounds for piecewise linear neural networks}{Bartlett, Harvey, Liaw and Mehrabian}
\firstpageno{1}

\begin{document}
\title{Nearly-tight VC-dimension and pseudodimension bounds \\
for piecewise linear neural networks\thanks{An extended abstract
appeared in Proceedings of the Conference on Learning Theory (COLT)
2017: \url{http://proceedings.mlr.press/v65/harvey17a.html}; the upper
bound was presented at the 2016 ACM Conference on Data Science:
\url{http://ikdd.acm.org/Site/CoDS2016/keynotes.html}.
This version includes all the proofs and
a refinement of the upper bound,
Theorem~\ref{newupperbound_averagedepth}.}}

\author{\name Peter L.~Bartlett \email bartlett@cs.berkeley.edu \\
       \addr Department of Statistics and Computer Science Division\\
       University of California\\
       Berkeley, CA 94720-3860, USA
\\
       \AND
       \name Nick Harvey \email nickhar@cs.ubc.ca \\
       \addr Department of Computer Science\\
       University of British Columbia\\
       Vancouver, BC V6T 1Z4, Canada\\
              \AND
              \name Christopher Liaw \email cvliaw@cs.ubc.ca \\
              \addr Department of Computer Science\\
              University of British Columbia\\
              Vancouver, BC V6T 1Z4, Canada
              \AND
              \name Abbas Mehrabian\thanks{corresponding author} \email abbasmehrabian@gmail.com \\
              \addr Department of Computer Science\\
              University of British Columbia\\
              Vancouver, BC V6T 1Z4, Canada
      }
      \editor{}
      
\maketitle

\begin{abstract}
We prove new upper and lower bounds on the VC-dimension of deep neural networks with the ReLU activation function.
These bounds are tight for almost the entire range of parameters.
Letting $W$ be the number of weights and $L$ be the number of layers,
we prove that the VC-dimension is $O(W L \log(W))$, and provide examples with VC-dimension $\Omega( W L \log(W/L) )$.
This improves both the previously known upper bounds and lower bounds.
In terms of the number $U$ of non-linear units, we prove a tight bound $\Theta(W U)$ on the VC-dimension.
All of these bounds generalize to arbitrary piecewise linear activation functions, and also hold for the pseudodimensions of these function classes.

Combined with previous results, this gives an
intriguing range of dependencies of the VC-dimension on depth for networks
with different non-linearities: there is no dependence for
piecewise-constant, linear dependence for piecewise-linear, and no
more than quadratic dependence for general piecewise-polynomial.
\end{abstract}

\begin{keywords}
VC-dimension, pseudodimension, neural networks, ReLU activation function, statistical learning theory
\end{keywords}

\section{Introduction}

Deep neural networks underlie many of the recent breakthroughs in applied
machine learning, particularly in image and speech recognition. 
These successes motivate a renewed study of these networks' theoretical properties. 

Classification is one of the learning tasks in which deep neural networks have been particularly successful, e.g., for image recognition.
A natural foundational question that arises is:
what are the generalization guarantees of these networks in a statistical learning framework?
An established way to address this question is by considering VC-dimension,
which characterizes uniform convergence
of misclassification frequencies to
probabilities~\citep[see][]{vc_originalpaper},
and 
asymptotically determines the sample complexity of PAC learning with such classifiers
\citep*[see][]{BEHW89}.

\begin{definition}[growth function, VC-dimension, shattering]
Let $H$ denote a class of functions from $\calX$ to  ${\{0,1\}}$ (the hypotheses, or the classification rules).
For any non-negative integer $m$, we define the \emph{growth function} of $H$ as
\[\growth_H(m)\coloneqq 
\max_{x_1,\dots,x_m \in \calX}
\left|\left\{(h(x_1),\ldots,h(x_m)): h\in H\right\}\right|.
\]
  If $\left|\left\{(h(x_1),\ldots,h(x_m)): h\in H\right\}\right|=2^m$, we say $H$ \emph{shatters} the set $\{x_1,\dots,x_m\}$.
  The Vapnik-Chervonenkis dimension of $H$, denoted $\VCdim(H)$,
  is the size of the largest shattered set, i.e.\ the largest $m$ such that $\growth_H(m)=2^m$.
  If there is no largest $m$, we define $\VCdim(H)=\infty$.
\end{definition}

For a class of real-valued functions,
such as those generated by neural networks,
a natural measure of complexity 
that implies similar uniform convergence
properties
is the pseudodimension \citep[see][]{p-epta-90}.

\begin{definition}[pseudodimension]
Let $\cF$ be a class of functions from $\calX$ to  $\Re$.
The pseudodimension of $\cF$, written $\pdim(\cF)$, is the largest integer $m$
for which there exists 
$(x_1,\dots,x_m,y_1,\dots,y_m)\in \calX^m \times \Re^m$
such that for any
$(b_1,\dots,b_m)\in \{0,1\}^m$
there exists $f\in\cF$ such that
\[
\forall i:
f(x_i)>y_i \Longleftrightarrow
b_i=1
\]
\end{definition}

For a class $\cF$ of real-valued functions,
we may define
$\VCdim(\cF)\coloneqq \VCdim(\sgn(\cF))$,
where \[
\sgn(\cF) \coloneqq \{\sgn(f) : f \in \cF\}
\]
and
$\sgn(x) = \mathbf{1}[x > 0]$.
For any class $\cF$, clearly $\VCdim(\cF)\leq\pdim(\cF)$.
If $\cF$ is the class of functions generated by a neural network with
a fixed architecture and fixed activation functions (see
Section~\ref{section:Notation} for definitions),
then it is not hard to see that indeed
$\pdim(\cF)\leq\VCdim(\cF)$ (see~\citep[Theorem~14.1]{AB99} for a proof), and hence
$\pdim(\cF)=\VCdim(\cF)$. 
Therefore, all the results of this paper automatically apply to the pseudodimensions of neural networks as well.

The main contribution of this paper is to prove nearly-tight bounds on the VC-dimension of deep neural networks
in which the non-linear activation function is a piecewise linear function with a constant number of pieces.
For simplicity we will henceforth refer to such networks as ``piecewise linear networks''.
The activation function that is the most commonly used in practice is
the \textit{rectified linear unit}, also known as \textit{ReLU}
 \citep*[see][]{LBG15,GBC16}.
The ReLU function  is defined as $\sigma(x)=\max \{0,x\}$, so it is clearly piecewise linear.

It is particularly interesting to consider how the VC-dimension is affected by the various attributes of the network:
the number $W$ of parameters (i.e., weights and biases),
the number $U$ of non-linear units (i.e., nodes),
and the number $L$ of layers.
Among all networks with the same size (number of weights), is it true that those with more layers have larger VC-dimension?

Such a statement is indeed true, and previously known;
however, a tight characterization of how depth affects VC-dimension was unknown prior to this work.

\subsection{Our results}
Our first main result is a new VC-dimension lower bound
that holds even for the restricted family of ReLU networks.

\begin{theorem}[Main lower bound]\label{splitintoblocks}
There exists a universal constant $C$ such that the following holds.
Given any $W,L$ with $W > CL > C^2$,
there exists a ReLU network with $\leq L$ layers and $\leq W$ parameters with VC-dimension $\geq WL \log(W/L)/C$.
\end{theorem}


\begin{remark}
Our construction can be augmented slightly to give a neural network with linear threshold and identity activation functions with the same guarantees.  
\end{remark}

The proof appears in~\cref{sec:splitintoblocks}.
Prior to our work, the best known lower bounds were $\Omega(WL)$
\citep*[see][Theorem~2]{BMM98} and
$\Omega(W\log W)$ \citep*[see][Theorem~1]{M94}.
We strictly improve both bounds to $\Omega(WL \log (W/L))$.

Our proof of Theorem~\ref{splitintoblocks} uses the ``bit extraction''
technique, which was also used in~\citep{BMM98} to give an $\Omega(WL)$ lower bound.
We refine this technique to gain the additional logarithmic factor that appears in Theorem~\ref{splitintoblocks}.

Unfortunately there is a barrier to refining this technique any further.
Our next theorem shows the hardness of computing the mod function, implying that the bit extraction technique cannot yield a stronger lower bound than Theorem~\ref{splitintoblocks}. Further discussion of this connection may be found in Remark~\ref{barrier}.

\begin{theorem}
\label{thm:bitextraction}
Assume there exists a piecewise linear network with $W$ parameters and $L$ layers that computes a function $f:\Re \to \R$, with the property that
$|f(x) - (x \mod 2)| < 1/2$ for all $x\in\{0,1,\dots,2^{m}-1\}$.
Then we have
$m = O(L \log(W/L)). $
\end{theorem}

The proof of this theorem appears in Section~\ref{sec:bitextractionbound}. 
One interesting aspect of the proof is that it does not use Warren's lemma \citep*{W67},
which is a mainstay of VC-dimension upper bounds \citep[see][]
{GJ95,BMM98,AB99}.

Our next main result is an upper bound on the VC-dimension of neural networks with piecewise polynomial activation functions.

\begin{theorem} [Main upper bound] \label{newupperbound_averagedepth}
	\label{thm:wllogw}
Consider a neural network architecture with $W$ parameters and $U$ computation units arranged in $L$ layers, so that each unit has connections only from units in earlier layers.
Let $k_i$ denote the number of  units at the $i$th layer.
Suppose that all non-output units have piecewise-polynomial
activation functions with $p+1$ pieces and degree no more than $d$,
and the output unit has the identity function as its activation function.

If $d=0$, let $W_i$
denote the number of parameters (weights and biases) at the inputs to
units in layer $i$;
if $d>0$,
let $W_i$
denote the total number of parameters (weights and biases) at the inputs to 
units in all the layers 
\textbf{up to layer $i$} (i.e., in layers $1,2,\dots,i$).
Define 
the {\em effective depth} as
  \[
    \bar L \coloneqq \frac{1}{W}\sum_{i=1}^L  W_i,
  \]
and let
  \begin{equation}\label{r-def}
    R \coloneqq \sum_{i=1}^{L} k_i (1+(i-1)d^{i-1}) \leq U+ U (L-1) d^{L-1}.
  \end{equation}
    For the class $\cF$ of all (real-valued) functions computed by this network and $m\geq \bar L W$, we have 
\[
 \growth_{\sgn(\cF)}(m) 
 \le \prod_{i=1}^{L}  2\left(\frac{2emk_ip(1+(i-1)d^{i-1})}{W_i}
            \right)^{ W_i}
            \leq 
            \left({ 4 e  mp  (1+(L-1)d^{L-1}) }\right)^{\sum W_i},
\]
and if $U>2$ then
\begin{align*}
  \VCdim(\cF) \le 
  L + \bar L W \log_2(4ep R \log_2 (2epR))
  = O(\bar L W \log (pU) + \bar L L W \log d).
\end{align*}
In particular, if $d=0$, then
\[
  \VCdim(\cF) \le 
  L + W \log_2(4ep U \log_2 (2epU))
  = O( W \log (pU));
\]
and if $d=1$, then
\[\VCdim(\cF) \leq 
  L + \bar L W \log_2(4ep \sum i k_i   \log_2 (\sum 2ep i k_i))
  = O(\bar L W \log (p U)).
  \]
  \end{theorem}

\begin{remark}
The average depth $\bar L$ is always between 1 and $L$, and captures how the parameters are distributed in the network:
it is close to $1$ if they are concentrated near the output
(or if the activation functions are piecewise-constant), while 
it is of order $L$ 
if the parameters
are concentrated near the input, or are spread out throughout
the network.
Hence, this suggests that edges and vertices closer to the input have a larger effect in increasing the VC-dimension,
a phenomenon not observed before;
and indeed our lower bound construction
in Theorem~\ref{splitintoblocks}
(as well as the lower bound construction from \citep{BMM98})
considers a network with most of the parameters near the input.
\end{remark}

The proof of this result appears in~\cref{sec:wllogw}.
Prior to our work, the best known upper bounds were 
$O(W^2)$ \citep[see][Section~3.1]{GJ95}
and
$O(W L \log W + W L^2)$ \citep[see][Theorem~1]{BMM98},
both of which hold for piecewise polynomial activation functions with a bounded number of pieces (for the remainder of this section, assume that $p=O(1)$ throughout);
we strictly improve both bounds to $O(WL \log W)$ for the special case of piecewise linear functions ($d=1$).
Recall that ReLU is an example of a piecewise linear activation function.
For the case $d=0$,
an $O(W \log U)$ bound for the VC-dimension was already proved using different techniques by
\citet{C68} and by \citet[Corollary 2]{BH89}.
Our \Cref{thm:wllogw} implies all of these upper bounds
(except the $O(W^2)$ upper bound of Goldberg and Jerrum) using a unified technique, and gives a slightly more
refined picture of the dependence of the VC-dimension on the
distribution of parameters in a deep network.

To compare our upper and lower bounds, 
let $d(W,L)$ denote the largest VC-dimension of a piecewise linear network with $W$ parameters and $L$ layers.
\cref{thm:wllogw,splitintoblocks} imply there exist constants $c,C$ such that
\begin{equation}
\label{eq:mainbounds}
c \cdot W L \log (W/L) ~\leq~ d(W,L) ~\leq~ C \cdot W L \log W\:.
\end{equation}
For neural networks arising in practice it would certainly be the case that $L$ is significantly smaller than $W^{0.99}$, in which case our results determine the asymptotic bound $d(W,L) = \Theta (W L \log W)$.
On the other hand, in the regime $L = \Theta(W)$, which is merely of theoretical interest, we also now have a tight bound $d(W,L)=\Theta(WL)$, obtained by combining \cref{splitintoblocks} with results of \citet{GJ95}.
There is now only a very narrow regime, say $W^{0.99} \ll L \ll W$, in which the bounds of \eqref{eq:mainbounds} are not asymptotically tight, and they differ only in the logarithmic factor.

Our final result is an upper bound for VC-dimension in terms of $W$ and $U$ (the number of non-linear units, or nodes).
This bound is tight in the case $d=1$ and $p=2$, as discussed in Remark~\ref{tightWU}.

\begin{theorem}
\label{thm:WU}
Consider a neural network with $W$ parameters and $U$ units with activation functions that are piecewise polynomials with at most $p$ pieces and of degree at most $d$.
    Let $\cF$ be the set of (real-valued) functions computed by this network.
    Then $\vcd(\sgn(\cF)) = O(W U \log ((d+1)p))$.
\end{theorem}

The proof of this result appears in Section~\ref{sec:wu}.
The best known upper bound before our work was $O(W^2)$, implicitly proven for bounded $d$ and $p$ by \citet[Section~3.1]{GJ95}.
Our theorem improves this to the tight result $O(WU)$.

We can summarize the tightest known results on the VC-dimension of neural networks with piecewise polynomial activation functions as follows: for classes $\cF$ of functions computed by
  the class of  networks with $L$ layers, $W$ parameters, and $U$ units with the following
  non-linearities, we have the following bounds on VC-dimension:
    \begin{description}
      \item [Piecewise constant.]
          $\VCdim(\cF) = \Theta(W \log W)$
        (\citet{C68} and \citet{BH89} showed the upper bound and
        \citet{M94} showed the lower bound).
      \item [Piecewise linear.]
      $c \cdot W L \log (W/L) ~\leq~ \VCdim(\cF) ~\leq~ C \cdot W L \log W$ (this paper).      
      \item [Piecewise polynomial.]
        $\VCdim(\cF) =  O(WL^2+WL\log W)$ \citep{BMM98}, and
        $\VCdim(\cF) =  O(WU)$ (this paper),
        and
        $ \VCdim(\cF) =\Omega(W L \log (W/L)) $ (this paper).
    \end{description}

\subsection{Related Work}
For other theoretical properties of neural networks, we refer the reader to the monograph~\citep{AB99}.
In this section, we summarize previous work that studies the impact of
depth on the representational power of neural networks. It has long
been known that two-layer networks with a variety of activation
functions can approximate arbitrary continuous functions on compact
sets~\citep{Hornik}. \citet{Sontag92} showed that three-layer networks
of linear threshold units can approximate inverses of continuous
functions, whereas two-layer networks cannot. There are several
recent papers that aim to understand which functions can be
expressed using a neural network of a given depth and size.
There are technical similarities between our work and these.
Two striking papers considered the  
problem of approximating a deep neural network with a shallower network.
\citet{T16} shows that there is a ReLU network with $L$ layers and $U = \Theta(L)$ units such that any network approximating it with only $O(L^{1/3})$ layers must have $\Omega(2^{L^{1/3}})$ units; this phenomenon holds even for real-valued functions.
\citet{ES16} show an analogous result for a high-dimensional 3-layer network that cannot be approximated by a 2-layer network except with an exponential blow-up in the number of nodes.

Very recently, several authors have shown that deep neural networks are capable of
approximating broad classes of functions.
\citet{SS16} show that a sufficiently non-linear $C^2$ function on $[0,1]^d$ can be approximated with $\epsilon$ error in $L_2$
by a ReLU network with $O(\polylog(1/\epsilon))$ layers and weights, but any such approximation with $O(1)$ layers requires $\Omega(1/\epsilon)$ weights.
\citet*{Y16} shows that any $C^n$-function on $[0,1]^d$ can be approximated with $\epsilon$ error
in $L_\infty$ by a ReLU network with $O(\log(1/\epsilon))$ layers and $O( (\frac{1}{\epsilon})^{d/n} \log(1/\epsilon) )$ weights.
\citet*{LS16} show that a sufficiently smooth univariate function can be approximated with $\epsilon$ error
in $L_\infty$ by a network with ReLU and threshold gates with $\Theta(\log(1/\epsilon))$ layers and 
$O(\polylog(1/\epsilon))$ weights, but that $\Omega(\poly(1/\epsilon))$ weights would be required if there were only $o(\log(1/\epsilon))$ layers; they also prove analogous results for multivariate functions.
Lastly, \citet*{CSS16} draw a connection to tensor factorizations to
show that, for a certain family of arithmetic circuits (in particular,
without ReLU non-linearities),
the set of functions computable by a shallow network have measure zero among those computable by a deep networks.

\subsection{Notation}\label{section:Notation}
A neural network is defined by an \textit{activation function}
$\psi : \Re \rightarrow \Re$,
a directed acyclic graph,
and a set of parameters:
a weight for each edge of the graph,
and a bias for each node of the graph.
Let $W$ denote the number of parameters (weights and biases) of the network, $U$ denote the number of computation units (nodes), and $L$ denote the length of the longest path in the graph.
We will say that the neural network has $L$ layers.

Layer 0 consists of nodes with in-degree $0$.
We call these nodes input nodes and they simply output the real value given by the corresponding input to the network.
We assume that the graph has a single sink node;
this is the unique node in layer $L$, which we call the output layer.
This output node can have predecessors in any layer $\ell < L$.
For $1 \leq \ell < L$, a node is in layer $\ell$ if it has a predecessor in layer $\ell-1$ and no predecessor in any layer $\ell' \geq \ell$.
(Note that for example there could be an edge connecting a node in layer 1 with a node in layer 3.)
In the jargon of neural networks, layers $1$ through $L-1$ are called hidden layers.

The computation of a neural network proceeds as follows.
For $i=1,\ldots,L$, the input into a computation unit $u$ at layer $i$ is $w^\top x+b$, where $x$ is the (real) vector corresponding to the outputs of the computational units with a directed
edge to $u$, $w$ is the corresponding vector of edge weights, and $b$ is the bias parameter associated with $u$.
For layers $1, \ldots, L-1$, the output of $u$ is $\psi(w^\top x + b)$.
For the output layer, we replace $\psi$ with the identity, so the output is simply $w^\top x + b$.
Since we consider VC-dimension, we will always take the sign of the output of the network, to make the output lie in $\{0,1\}$ for binary classification.

A piecewise polynomial function with $p$ pieces is a function $f$ for
which there exists a partition of $\Re$ into disjoint intervals
(pieces) $I_1, \ldots, I_p$ and corresponding polynomials $f_1, \ldots, f_p$
such that if $x \in I_i$ then $f(x) = f_i(x)$.
A piecewise linear function is a piecewise polynomial function in which each $f_i$ is linear.
The most common activation function used in practice is the rectified linear unit (ReLU) where $I_1 = (-\infty, 0]$, $I_2 = (0, \infty)$ and $f_1(x) = 0, f_2(x) = x$.
We denote this function by $\sigma(x) \coloneqq \max\{0,x\}$.
The set $\{1,2,\dots,n\}$ is denoted $[n]$.

\section{Proof of Theorem~\ref{splitintoblocks}}
\label{sec:splitintoblocks}

The proof of our main lower bound uses the ``bit extraction'' technique
from \citep{BMM98} to prove an $\Omega(WL)$ lower bound.
We refine this technique in a key way --- we partition the input bits into blocks and extract multiple bits at a time instead of a single bit at a time.
This yields a more efficient bit extraction network, and hence a stronger VC-dimension lower bound.

We show the following result, which immediately implies \Cref{splitintoblocks}.
\begin{theorem}
\label{thm:newlowerbound}
Let $r,m,n$ be positive integers,
and let $k = \lceil m/r \rceil$.
There exists a ReLU network with $3+5k$ layers,
$2+n+4m+k((11+r)2^r+2r+2)$ parameters,
$m+n$ input nodes
and $m+2+k(5 \times 2^{r}+r+1)$ computational nodes
with VC-dimension $\geq mn$.
\end{theorem}

\begin{remark}\label{tightWU}
Choosing $r=1$ gives a network with
$W=O(m+n)$, $U=O(m)$ and VC-dimension $\Omega(mn)=\Omega(WU)$.
This implies that the upper bound $O(WU)$ given in~\cref{thm:WU} is tight.
\end{remark}

To prove \Cref{splitintoblocks},
assume $W$, $L$, and $W/L$ are sufficiently large, and set 
$r = \log_2(W/L)/2$, 
$m = rL/8$, and $n = W - 5m2^r$ in \Cref{thm:newlowerbound}.
The rest of this section is devoted to proving \Cref{thm:newlowerbound}.

Let $S_n \subseteq \Re^n$ denote the standard basis.
We shatter the set $S_n \times S_m$.
Given an arbitrary function $f \colon S_n \times S_m \to \{0,1\}$,
we build a ReLU neural network that takes as input $(x_1,x_2)\in S_n \times S_m$ and outputs $f(x_1,x_2)$.
Define $n$ numbers 
$a_1,a_2,\dots,a_n \in \{\frac{0}{2^m},\frac{1}{2^m}, \dots, \frac{2^m-1}{2^m}\}$
so that the $i$th digit of the binary representation of $a_j$ equals $f(e_j,e_i)$.
These numbers will be used as the parameters of the network, as described below.

Given input $(x_1,x_2) \in S_n \times S_m$,
assume that $x_1 = e_i$ and $x_2 = e_j$.
The network must output the $i$th bit of $a_j$.
This ``bit extraction approach'' was used
in \citep[Theorem~2]{BMM98} to give an $\Omega(WL)$ lower bound for the VC-dimension. We use a similar approach but we introduce a novel idea: we split the bit extraction into blocks and extract $r$ bits at a time instead of a single bit at a time.
This allows us to prove a lower bound of $\Omega(WL \log (W/L))$.
One can ask, naturally, whether this approach can be pushed further. Our \Cref{thm:bitextraction} implies that the bit extraction approach cannot give a lower bound better than 
$\Omega(WL \log (W/L))$ (see Remark~\ref{barrier}).

The first layer of the network
``selects''  $a_j$,
and the remaining layers ``extract'' the $i$th bit of $a_j$.
In the first layer we have a single computational unit that calculates 
$$a_j = (a_1, \ldots, a_n)^\top x_1 =\sigma\left( (a_1, \ldots, a_n)^\top x_1 \right).$$
This part uses 1 layer, 1 computation unit, and $1+n$ parameters.

The rest of the network extracts all bits of $a_j$ and outputs the $i$th bit.
The extraction is done in $k$ steps, where in each step we extract the $r$ most significant bits and zero them out. 
We will use the following building block for extracting $r$ bits.

\begin{lemma}
Suppose positive integers $r$ and $m$ are given.
There exists a ReLU network with 5 layers, $5\times 2^{r}+r+1$ units and $11\times 2^{r}+r2^r+2r+2$ parameters that given the real number $b = 0.b_1 b_2 \dots b_m$ (in binary representation) as input, 
outputs  the $(r+1)$-dimensional vector
$(b_1,b_2,\dots,b_r,0.b_{r+1}b_{r+2}\dots b_m)$.
\end{lemma}
\Cref{fig:bit_extract} shows a schematic of the ReLU network in the above lemma.
\begin{figure}[h]
\centering
\includegraphics[width=15cm]{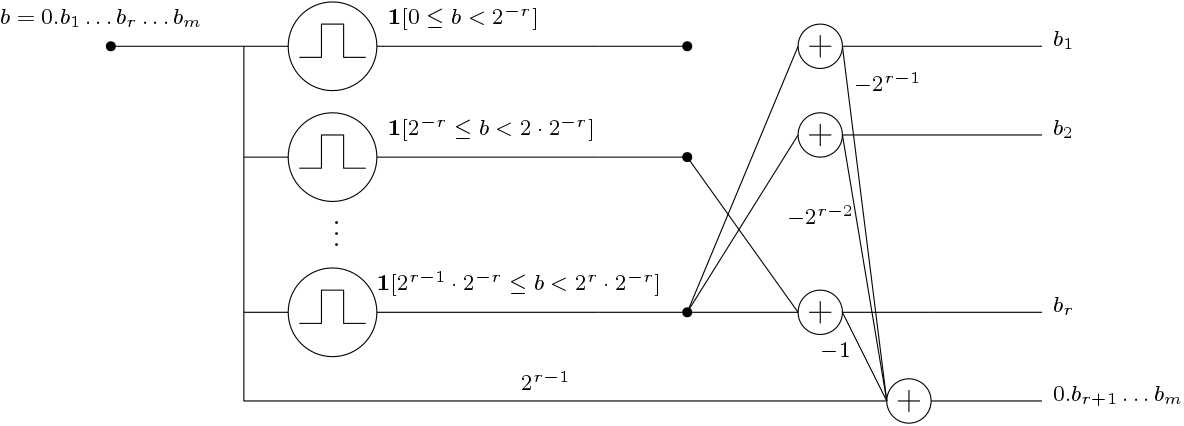}
\caption{The ReLU network used to extract the most significant $r$ bits of a number.
Unlabeled edges indicate a weight of 1 and missing edges indicate a weight of 0.}
\label{fig:bit_extract}
\end{figure}

\begin{proof}
Partition $[0,1)$ into $2^{r}$ even subintervals.
Observe that the values of $b_1,\dots,b_r$ are determined by knowing which such subinterval $b$ lies in.
We first show how to design a two-layer ReLU network that computes the indicator function for an interval to any arbitrary precision.
Using $2^r$ of these networks in parallel allows us to determine which subinterval $b$ lies in and hence, determine
the bits $b_1,\dots,b_r$.

For any $a \leq b$ and $\eps > 0$, observe that the function
$f(x) \coloneqq \sigma(1-\sigma(a/\eps-x/\eps))+\sigma(1-\sigma(x/\eps-b/\eps))-1$
has the property that,
$f(x)=1$ for $x\in[a,b]$,
and $f(x)=0$ for $x\notin(a-\eps,b+\eps)$,
and $f(x)\in[0,1]$ for all $x$.
Thus we can use $f$ to approximate the indicator function for $[a,b]$, to any desired precision.
Moreover, this function can be computed with 3 layers, 5 units, and 11 parameters as follows.
First, computing $\sigma(a/\eps-x/\eps)$ can be done with 1 unit, 1 layer, and 2 parameters.
Computing $\sigma(1-\sigma(a/\eps-x/\eps))$ can be done with 1 additional unit, 1 additional layer, and 2 additional parameters.
Similarly, $\sigma(1-\sigma(x/\eps-b/\eps))$ can be computed with 2 units, the same 2 layers, and 4 parameters.
Computing the sum can be done with 1 additional layer, 1 additional unit, and 3 additional parameters.
In total, computing $f$ can be done with 3 layers, 5 units, and 11 parameters.
We will choose $\eps = 2^{-m-2}$ because we are working with $m$-digit numbers.

Thus, the values $b_1,\dots,b_r$ can be generated by adding the corresponding indicator variables.
(For instance, $b_1 = \sum_{k = 2^{r-1}}^{2^r - 1} \mathbf{1}[b \in [k \cdot 2^{-r}, (k+1) \cdot 2^{-r})]$.)
Finally, the remainder $0.b_{r+1}b_{r+2}\dots b_{m}$ can be computed as $0.b_{r+1}b_{r+2} \ldots b_{m} = 2^{r} b - \sum_{k=1}^r 2^{r-k} b_k$.

Now we count the number of layers and parameters: we use $2^r$ small networks that work in parallel for producing the indicators, each has 3 layers, 5 units and 11 parameters.
To produce $b_1,\dots,b_r$ we need an additional layer, $r\times (2^r+1)$ additional parameters,
and $r$ additional units.
For producing the remainder we need 1 more layer,
1 more unit, and $r+2$ more parameters.
\end{proof}

We use $\lceil m/r \rceil$ of these blocks to extract the bits of $a_j$, denoted by  $a_{j,1},\dots,a_{j,m}$.
Extracting $a_{j,i}$ is now easy, noting that if $x,y\in\{0,1\}$ then $x \wedge y = \sigma(x+y-1)$.
So, since $x_2=e_i$, we have
\[
a_{j,i} = \sum_{t=1}^{m} x_{2,t} \wedge a_{j,t}
= \sum_{t=1}^{m} \sigma(x_{2,t} + a_{j,t}-1)
= \sigma\left(\sum_{t=1}^{m} \sigma(x_{2,t} + a_{j,t}-1)\right).
\]
This calculation needs 2 layers, $1+m$ units, and $1+4m$ parameters.

\begin{remark}
\label{barrier}
\Cref{thm:bitextraction} implies an inherent barrier to proving lower
bounds using the ``bit extraction'' approach from \citep{BMM98}.
Recall that this technique uses $n$ binary numbers with $m$ bits to encode a function $f \colon S_n \times S_m \to \{0,1\}$
to show an $\Omega(mn)$ lower bound for VC-dimension, where $S_k$ denotes the set of standard basis vectors in $\Re^k$.
The network begins by selecting one of the $n$ binary numbers, and then extracting a particular  bit of that number.
\citep{BMM98} shows that it is possible to take $m=\Omega(L)$ and  $n=\Omega(W)$, thus proving a lower bound of $\Omega(WL)$ for the VC-dimension.
In \Cref{splitintoblocks} we showed we can increase $m$ to $\Omega(L \log(W/L) )$, improving the lower bound  to $\Omega(WL \log(W/L))$.
\Cref{thm:bitextraction} implies that 
to extract just the least significant bit, one is forced to have $m=O(L \log (W/L))$;
on the other hand, we always have $n\leq W$.
Hence there is no way to improve the VC-dimension lower bound by more than a constant via the bit extraction technique.
In particular, for general piecewise polynomial networks, closing the
gap between the $O(WL^2+WL\log W)$ of \citep{BMM98} and $\Omega(WL\log
W/L)$ of this paper will require a different technique.
\end{remark}

\section{Proof of Theorem~\ref{thm:bitextraction}}
\label{sec:bitextractionbound}

For a piecewise polynomial function $\Re \to \Re$,
\emph{breakpoints} are the boundaries between the pieces.
So if a function has $p$ pieces, it has $p-1$ breakpoints.

\begin{lemma}
\label{lem:breakpointgrowth}
Let $f_1,\dots,f_k:\R\to\R$ be piecewise polynomial of degree $D$,
and suppose the union of their breakpoints has size $B$.
Let $\psi:\R\to\R$ be piecewise polynomial of degree $d$ with $b$ breakpoints.
Let $w_1,\dots,w_k\in\R$ be arbitrary.
The function $g(x)\coloneqq\psi(\sum_i w_i f_i(x))$ is piecewise polynomial of degree $Dd$ with at most
$(B+1)(2+bD)-1$ breakpoints.
\end{lemma}

\begin{proof}
Without loss of generality, assume that $w_1=\dots=w_k=1$.
The function $\sum_i f_i$ has $B+1$ pieces.
Consider one such interval $\mathcal I$.
We will prove that it will create at most $2+bD$ pieces in $g$.
In fact, if $\sum_i f_i$ is constant on $\mathcal I$, $g$ will have 1 piece on $\mathcal I$.
Otherwise, for any point $y$, the equation $\sum_i f_i(x) = y$ has at most $D$ solutions on $\mathcal I$. 
Let $y_1,\dots,y_b$ be the breakpoints of $\psi$.
Suppose we move along the curve $(x,\sum_i f_i(x))$ on $\mathcal I$.
Whenever we hit a point $(t,y_i)$ for some $t$, one new piece is created in $g$.
So at most $bD$ new pieces are created.
In addition, we may have two pieces for the beginning and ending of $\mathcal I$. This gives a total of $2+bD$ pieces per interval, as required.
Finally, note that the number of breakpoints is one fewer than the number of pieces.
\end{proof}

Theorem~\ref{thm:bitextraction} follows immediately from the following theorem.

\begin{theorem}
\label{thm:bitextract_barrier}
Assume there exists  a neural network with $W$ parameters and $L$ layers that computes a function $f:\R \to \R$, with the property that
$|f(x) - (x \mod 2)| < 1/2$ for all $x\in\{0,1,\dots,2^{m}-1\}$.
Also suppose the activation functions are piecewise polynomial of degree at most $d\geq1$ in each piece, and have at most $p\geq1$ pieces.
Then we have
$$m \leq L \log_2 (13pd^{(L+1)/2}\cdot W/L). $$
In the special case of piecewise linear functions, this gives
$m = O(L \log(W/L)). $
\end{theorem}

\begin{proof}
For a node $v$ of the network, let $\gamma(v)$ count the number of directed paths from the input node to $v$.
Applying Lemma~\ref{lem:breakpointgrowth} iteratively gives that for a node $v$ at layer $i\geq 1$, the number of breakpoints is bounded by
$(6p)^{i} d^{i(i-1)/2} \gamma(v)-1$.
Let $o$ denote the output node.
Hence, $o$ has at most $(6p)^{L} d^{L(L-1)/2} \gamma(o)$ pieces.
The output of node $o$ is piecewise polynomial of degree at most $d^L$.
On the other hand, as we increase $x$ from 0 to $2^m-1$, the 
function $x \mod 2$ flips  $2^m-1$ many times, which implies the output of $o$ becomes equal to $1/2$ at least $2^m-1$ times, thus we get
\begin{equation}
\label{1}
(6p)^{L} d^{L(L-1)/2} \gamma(o) \times d^L \geq 2^m - 1.
\end{equation}
Let us now relate $\gamma(o)$ with $W$ and $L$.
Suppose that, for $i\in[L]$, there are $W_i$ edges between layer $i$ and previous layers.
By the AM-GM inequality,
\begin{equation}
\label{2}
\gamma(o) \leq \prod_i (1+W_i) \leq \left(\sum_i \frac{1+W_i}{L}\right)^L \leq (2W/L)^L.
\end{equation}
Combining~\cref{1,2} gives the theorem.
\end{proof}

\citet{T16} showed how to construct a function $f$ which satisfies $f(x) = (x \mod 2)$ for $x \in \{0,1, \ldots, 2^m-1\}$ using a neural network with $O(m)$ layers and $O(m)$ parameters.
By choosing $m = k^3$, Telgarsky showed that any function $g$ computable by a neural network with $\Theta(k)$ layers and $O(2^k)$ nodes must necessarily have $\|f - g\|_1 > c$ for some constant $c > 0$.

Our theorem above implies a qualitatively similar statement.
In particular, if we choose $m = k^{1+\eps}$ then for any function $g$ computable by a neural network with $\Theta(k)$ layers and $O(2^{k^\eps})$ parameters, there must exist $x \in \{0, 1, \ldots, 2^m - 1\}$ such that $|f(x) - g(x)| > 1/2$.

\section{Proof of Theorem~\ref{thm:wllogw}}
\label{sec:wllogw}
The proof of this theorem is very similar to the proof of the upper
bound for piecewise polynomial networks from \citep[Theorem~1]{BMM98} but optimized in a few places.
The main technical tool in the proof is a bound on the growth function
of a polynomially parametrized function class, due to \citet{GJ95}. It uses an argument
involving counting the number of connected components of
semi-algebraic sets.
The form stated here 
is \citep[Lemma~1]{BMM98}, which is
a slight improvement of a result of \citet{W67} (the proof can be
found in \citep[Theorem~8.3]{AB99}).

\begin{lemma}\label{lem:warren} 
	Let $p_1, \ldots, p_m$ be polynomials of degree at most $d$ in $n \leq m$ variables.
    Define
    \[
    	K \coloneqq |\{(\sgn(p_1(x)), \ldots, \sgn(p_m(x)) : x \in \R^n \}|,
    \]
    i.e.~$K$ is the number of possible sign vectors given by the polynomials.
    Then $K \leq 2(2emd/n)^n$.
\end{lemma}

\begin{proof}[of \Cref{thm:wllogw}].
For input $x\in\calX$ and parameter vector $a\in \Re^W$,
let $f(x,a)$ denote the output of the network.
The $\cF$ is simply the class of functions $\{x\mapsto f(x,a): a\in\Re^W\}$.

Fix $x_1,x_2,\ldots,x_m$ in $\calX$.
We view the parameters of the network, denoted $a$, as a collection of $W$ real variables.
We wish to bound
\[
K \coloneqq \left|\left\{
 (\sgn(f(x_1,a)),\ldots,\sgn(f(x_m,a))) : a\in \Re^W
\right\}\right|.
\]
In other words, $K$ is the number of sign patterns that the neural network can output for the sequence of inputs $(x_1, \ldots, x_m)$.
We will prove geometric upper bounds for $K$, which will imply upper bounds for $\growth_{\sgn(\cF)}(m)$

For any partition
$\calS= \{P_1,P_2,\ldots,P_N\}$ of the parameter domain $\Re^W$,
clearly we have
\begin{equation}\label{unionbound}
K \le \sum_{i=1}^N \left|\left\{(\sgn(f(x_1,a)),\ldots,\sgn(f(x_m,a))):
	a\in P_i\right\}\right|.
\end{equation}
We choose the partition in such a way that within each region $P_i$,
the functions $f(x_j,\cdot)$ are all fixed
polynomials of bounded degree, so that each term in this sum can be bounded via Lemma~\ref{lem:warren}.

The partition is constructed iteratively layer by layer,
through a sequence $\calS_0,\calS_1,\calS_2,\ldots,\calS_{L-1}$
of successive refinements, with the following properties:

\begin{enumerate}
\item
We have $|\calS_0|=1$
and, for each $n\in[L-1]$,
\begin{equation}\label{partitions}
\frac{|\calS_n|}{|\calS_{n-1}|} \leq 
2 \left( \frac{2 e m k_n p (1 + (n-1)d^{n-1}) }{W_n}\right)^{W_n}
\end{equation}

\item
For each $n\in\{0,\dots,L-1\}$,
each element $S$ of $\calS_{n-1}$,
each $j\in [m]$,
and each unit $u$ in the $n$th layer,
when $a$ varies in $S$, 
the net input to $u$ is a fixed polynomial function 
in $W_n$ variables of $a$,
of total degree no more than $1 + (n-1)d^{n-1}$
(this polynomial may depend on $S$, $j$ and $u$).
\end{enumerate}

We may define $\calS_0=\Re^{W}$,
which satisfies property 2 above,
since the input to any node in layer 1 is of the form $w^Tx_j + b$, which is an affine function of $w,b$.

Now suppose that 
$\calS_0,\dots,\calS_{n-1}$ have been defined,
and we want to define $\calS_n$.
For any $h\in[k_n],j\in[m],$ and $S\in \calS_{n-1}$,
let $p_{h,x_j,S}(a)$ denote the function describing the net input
of the $h$-th unit in the $n$-th layer, in response to $x_j$,
when $a\in S$.
By the induction hypothesis this is a polynomial with total degree no more than $1+(n-1)d^{n-1}$,
and depends on at most $W_{n}$ many variables.

Let $\{t_1,\dots,t_p\}$ denote the set of breakpoints of the activation function.
For any fixed $S\in\calS_{n-1}$,
by Lemma~\ref{lem:warren},
the collection of polynomials
\[\big\{ 
p_{h,x_j,S}(a)-t_i
 : h \in [k_n], j\in[m], i \in[p] \} \]
attains at most 
\[\Pi\coloneqq 2 (2 e (k_n m p)(1+(n-1)d^{n-1}) /W_n)^{W_n}\] 
distinct
sign patterns when $a \in \Re^W$.
Thus, one can partition $\Re^W$ into this many regions,
such that all these polynomials have the same signs within each region.
We intersect all these regions with $S$
to obtain a partition of $S$ into at most
$\Pi$
subregions.
Performing this for all $S\in\calS_{n-1}$ gives our desired partition $\calS_{n}$.
Thus, the required property 1 (inequality (\ref{partitions})) is clearly satisfied.

Fix some $S'\in \calS_n$.
Notice that, when $a$ varies in $S'$,
all the polynomials
\[\big\{ 
p_{h,x_j,S}(a)-t_i
 : h \in [k_n], j\in[m], i \in[p] \} \]
have the same sign, hence the 
\emph{input} of each $n$th layer unit lies between two breakpoints of the activation function, hence the
\emph{output} of each $n$th layer unit in response to an
$x_j$ is a fixed polynomial in $W_n$ variables
of degree no more than $d(1+(n-1)d^{n-1})\le n d^n$.
This implies that the \emph{input} of every $(n+1)$th layer unit  in response to an $x_j$
is a fixed polynomial function of $W_{n+1}$ variables of
 degree no more than $1+nd^n$. 
(When $d=0$, this
affine function depends only on the $W_{n+1}$ parameters in layer
$n+1$; for $d>0$, it is a polynomial function of all parameters up to
layer $n+1$.)

Proceeding in this way we obtain a partition $\calS_{L-1}$ of $\Re^W$ such that
for $S\in\calS_{L-1}$ the network output in response to any $x_j$ is a 
fixed polynomial of $a\in S$ of degree no more than
$1+(L-1)d^{L-1}$ (recall that the last node just outputs its input),
and hence by Lemma~\ref{lem:warren} again,
  \begin{align*}
    \left|\left\{\left(\sgn(f(x_1,a)),\ldots,\sgn(f(x_m,a))\right):
    a\in S\right\}\right|
    & \le 2 \left(\frac{2em(1+(L-1)d^{L-1})}{ W_L}\right)^{ W_L}.
  \end{align*}
On the other hand, applying (\ref{partitions}) iteratively gives
\[|\calS_{L-1}|\leq
	\prod_{i=1}^{L-1}
        2\left(\frac{2emk_ip(1+(i-1)d^{i-1})}{ W_i}
			\right)^{ W_i},\]
and thus using (\ref{unionbound}),
and since the points $x_1,\ldots,x_m$ were chosen arbitrarily,
we obtain
\begin{align*}
 \growth_{\sgn(\cF)}(m) 
 & \le \prod_{i=1}^{L}  2\left(\frac{2emk_ip(1+(i-1)d^{i-1})}{ W_i}
            \right)^{ W_i} \\
            & \leq
            2^L\left(\frac{2 e m p \sum k_i(1+(i-1)d^{i-1})}{\sum W_i}\right)^{\sum W_i} & \textnormal{(weighted AM-GM)} \\
                        & =
                        2^L\left(\frac{2 e m p R}{\sum W_i}\right)^{\sum W_i} &\textnormal{(definition of $R$ in (\ref{r-def}))}\\
                        & \leq
                        \left(\frac{4 e m p (1+(L-1)d^{L-1}) \sum k_i}{\sum W_i}\right)^{\sum W_i} & (L\leq \sum W_i)\\
                        & \leq \left(4 e m p (1+(L-1)d^{L-1}) \right)^{\sum W_i} & (\sum k_i \leq \sum W_i).
\end{align*}
For the bound on the VC-dimension,
from the third line in the formula above, and the definition of VC-dimension, we find
\[
2^{\VCdim(\cF)} =
 \growth_{\sgn(\cF)}(\VCdim(\cF)) 
\leq            2^L\left(\frac{2 e  p R \cdot \VCdim(\cF)}{\sum W_i}\right)^{\sum W_i}  \]
Notice that $U>2$ implies $2eR\geq 16$,
hence Lemma~\ref{growthtovc} below gives
\begin{align*}
  \VCdim(\cF) \le 
  L + (\sum W_i) \log_2(4ep R \log_2 (2epR))
  = O(\bar L W \log (pU) + \bar L L W \log d),
\end{align*}
completing the proof.
%
%
\end{proof}

\begin{lemma}\label{growthtovc}
Suppose that $2^m \leq 2^t (mr/w)^w$
for some $r\geq16$ and $m \geq w\geq t \geq 0$.
Then, $m \leq t + w \log_2 (2r \log_2 r)$.
\end{lemma}
\begin{proof}
We would like to show that
$2^x > 2^t (xr/w)^w$ for all 
$x>t + w \log_2 (2r \log_2 r) \eqqcolon m$.
Let $f(x) \coloneqq  x - t - w \log_2(xr/w)$.
To show that $f(x)>0$ for all $x>m$, we need only show that  $f(m)\geq0$ and $f'(x)>0$ for all $x\geq m$.
First, $f(m)\geq0$ if and only if
\[
w \log_2 (2r \log_2 r)  - w \log_2(mr/w)\geq0,
\]
if and only if
\[
(2r \log_2 r)  - (mr/w)\geq0,
\]
if and only if
\[
2 \log_2 r  - (t + w \log_2 (2r \log_2 r))/w\geq0,
\]
if and only if
\[
2 \log_2 r  - t/w - \log_2 (2r \log_2 r)\geq0,
\]
if and only if
\[r^2 / 2 \log_2 r \geq 2^{t/w},\]
which holds since $r\geq 16$ and $t/w\leq 1$.
Finally, for $x\geq m$, we have $f'(x)\geq0$ if and only if
\[
1 - w /(x\ln(2)) \geq 0
\]
if and only if
\[x \geq w / \ln 2,\]
which holds since $r\geq16$ implies
$
x \geq m \geq w \log_2(2r \log_2 r) >  w/\ln 2$.
\end{proof}

\section{Proof of Theorem~\ref{thm:WU}}
\label{sec:wu}
The idea of the proof is that the sign of the output of a neural network can be expressed as a Boolean formula where each predicate is a polynomial inequality.
For example, consider the following toy network, where the activation function of the hidden units is a ReLU.
\begin{figure}[h]
\centering
\includegraphics[width=4cm]{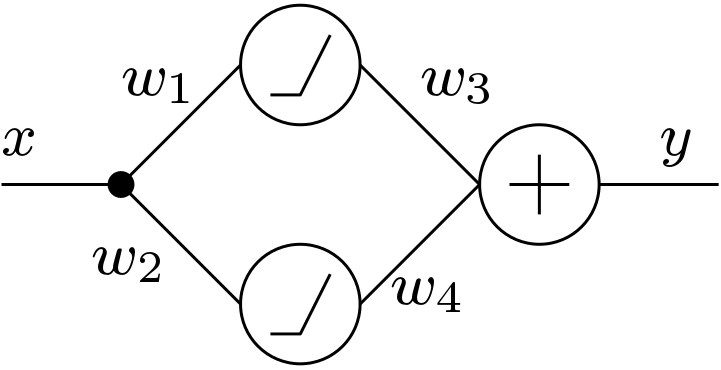}
\end{figure}

The sign of the output of the network is $\sgn(y) = \sgn(w_3 \sigma(w_1x) + w_4 \sigma(w_2x))$.
Define the following Boolean predicates: $p_1 = (w_1x > 0)$, $p_2 = (w_2x > 0)$, $q_1 = (w_3w_1x > 0)$, $q_2 = (w_4w_2x > 0)$, and $q_3 = (w_3w_1 x + w_4w_2 x > 0)$.
Then, we can write
\[
\sgn(y) = (\neg p_1 \wedge \neg p_2 \wedge 0) \vee (p_1 \wedge \neg p_2 \wedge q_1) \vee (\neg p_1 \wedge p_2 \wedge q_2) \vee (p_1 \wedge p_2 \wedge q_3).
\]

A theorem of Goldberg and Jerrum states that any class of functions that can be expressed using a relatively small number of distinct polynomial inequalities has small VC-dimension.

\begin{theorem}[Theorem~2.2 of \citet{GJ95}]
	\label{thm:goldberg}
	Let $k,n$ be positive integers and $f \colon \R^n \times \R^k \to \{0,1\}$ be a function that
    can be expressed as a Boolean formula containing $s$ distinct atomic predicates
    where each atomic predicate is a polynomial inequality or equality in $k + n$ variables of degree at most $d$.
    Let $\cF = \{f(\cdot, w) : w \in \R^k\}$.
    Then $\vcd(\cF) \leq 2k \log_2(8eds)$.
\end{theorem}
\begin{proof}[of \cref{thm:WU}].
	Consider a neural network with $W$ weights and $U$ computation units, and assume that the activation function $\psi$ is piecewise polynomial of degree at most $d$ with $p$ pieces.
    To apply \cref{thm:goldberg}, we will express the sign of the output of the network as a Boolean function consisting of less than $2 (1+p)^{U}$ atomic predicates, each being a  polynomial inequality of degree at most $\max\{U+1, 2d^{U}\}$.

    Since the neural network graph is acyclic, it can be topologically sorted.
    For $i \in [U]$, let $u_i$ denote the $i$th computation unit in the topological ordering.
The input to each computation unit $u$ lies in one of the $p$ pieces of $\psi$. For $i\in[U]$ and $j\in[p]$, we say ``$u_i$ is in state $j$'' if the input to $u_i$ lies in the $j$th piece.

For $u_1$ and any $j$, 
the predicate ``$u_1$ is in state $j$'' is a single atomic predicate which is the quadratic inequality indicating whether its input lies in the corresponding interval.
So, the state of $u_1$ can be expressed as a function of $p$ atomic predicates.
Conditioned on $u_1$ being in a certain state,
the state of $u_2$ can be determined using $p$ atomic predicates, which are polynomial  inequalities of degree at most $2d+1$.
Consequently, the state of $u_2$ can be  determined using $p+p^2$ atomic predicates, each of which is a polynomial of degree at most $2d+1$.
Continuing similarly, we obtain that for each $i$, the state of $u_i$ can be determined using $p(1+p)^{i-1}$ atomic predicates, each of which is a polynomial of degree at most $d^{i-1} + \sum_{j=0}^{i-1} d^j$.
Consequently, the state of all nodes can be determined using 
less than $(1+p)^{U}$ atomic predicates, each of which is a polynomial of degree at most $d^{U-1} + \sum_{j=0}^{U-1} d^j \leq \max\{U+1, 2d^{U}\}$ (the output unit is linear).
Conditioned on all nodes being in certain states,
the sign of the output can be determined using one more atomic predicate, which is a  polynomial  inequality  of degree at most $\max\{U+1, 2d^{U}\}$.

    In total, we have less than $2 (1+p)^{U}$ atomic polynomial-inequality predicates and each polynomial has degree at most $\max\{U+1, 2d^{U}\}$.
    Thus, by \cref{thm:goldberg}, we get an upper bound of $2W \log(16 e \cdot \max\{U+1, 2d^{U}\} \cdot (1+p)^{U}) = O(WU \log ((1+d)p))$ for the VC-dimension.
\end{proof}

  \acks{Christopher Liaw is supported by an NSERC graduate scholarship.
Abbas Mehrabian is supported by an NSERC Postdoctoral Fellowship and a
Simons-Berkeley Research Fellowship.
Peter Bartlett gratefully acknowledges the support of the
NSF through grant IIS-1619362 and of the Australian Research Council
through an Australian Laureate Fellowship (FL110100281) and through
the Australian Research Council Centre of Excellence for Mathematical
and Statistical Frontiers (ACEMS).
Part of this work was done while Peter Bartlett and Abbas Mehrabian
were visiting the Simons Institute for the Theory of Computing at UC
Berkeley.}

\end{document}